\documentclass[letterpaper]{article} 
\usepackage{aaai25}  
\usepackage{times}  
\usepackage{helvet}  
\usepackage{courier}  
\usepackage[hyphens]{url}  
\usepackage{graphicx} 
\usepackage{natbib}  
\usepackage{caption} 
\usepackage{algorithm}
\usepackage{xspace}
\usepackage{algorithmic}
\usepackage{mathtools}
\usepackage{amsfonts}
\usepackage{amssymb}
\usepackage{breqn}
\usepackage{enumerate}
\usepackage{enumitem}
\usepackage{physics}
\usepackage{pifont}
\usepackage{amsmath}
\usepackage{amsthm}
\usepackage{adjustbox}
\usepackage{booktabs}
\usepackage{multirow}
\usepackage{breakurl}
\usepackage{color}
\usepackage{array,graphicx}
\usepackage{pifont}
\usepackage{colortbl}
\usepackage{caption}
\newcommand{\TSK}[1]{}

\newtheorem{problem}{Problem}

\newtheorem{definition}{Definition}
\newtheorem{proposition}{Proposition}

\newcommand{\mypara}[1]{\noindent{\textbf{#1.}}}

\newcommand{\method}{\text{LaNoLem}\xspace}
\newcommand{\methodn}{\text{LaNoLem-N}\xspace}
\newcommand{\methodl}{\text{LaNoLem-L}\xspace}
\newcommand{\mSet}{\mathbf{M}\xspace}
\newcommand{\fitalgo}{\textit{Optimization algorithm}\xspace}
\newcommand{\sfit}{Inference\xspace}
\newcommand{\pfit}{Learning\xspace}

\newcommand{\datalen}{N\xspace}
\newcommand{\datadim}{d\xspace}
\newcommand{\sdim}{k\xspace}
\newcommand{\nldim}{d_\phi\xspace}
\newcommand{\nlsdim}{\sdim_\phi\xspace}

\newcommand{\data}{\mathbf{X}}
\newcommand{\State}{\mathbf{S}}
\newcommand{\datapoint}{x}
\newcommand{\datavec}{\mathbf{\datapoint}}
\newcommand{\statepoint}{s}
\newcommand{\statevec}{\mathbf{\statepoint}}
\newcommand{\nlstatevec}{\boldsymbol \phi(\statevec(t), \nldim)}
\newcommand{\nonlinearities}{[\statevec^{\ast 2}(t),...,\statevec^{\ast \nldim}(t)]}

\newcommand{\estpoint}{y}
\newcommand{\estvec}{\mathbf{\estpoint}}

\newcommand{\lMat}{\mathbf{A}}
\newcommand{\nlMat}{\mathbf{F}}
\newcommand{\sOf}{\mathbf{b}}
\newcommand{\oMat}{\mathbf{C}}
\newcommand{\oOf}{\mathbf{u}}
\newcommand{\eye}{\mathbf{I}}

\newcommand{\jMat}{\mathbf{J}}

\newcommand{\sCov}{\mathbf{\Gamma}}
\newcommand{\oCov}{\mathbf{R}}

\newcommand{\params}{\{\lMat,\nlMat,\sOf,\oMat,\oOf\}}
\newcommand{\sparseparams}{\{\lMat,\nlMat\}}
\newcommand{\nonsparseparams}{\{\sOf, \oMat, \oOf\}}
\newcommand{\setparams}{\theta}
\newcommand{\setsparse}{\theta_{s}}
\newcommand{\sparseMat}{\mathbf{\Theta}_s}
\newcommand{\setnonsparse}{\theta_{ns}}

\newcommand{\nll}{Q}

\newcommand{\msCov}{\mathbf{P}}
\newcommand{\msB}{\mathbf{V}}
\newcommand{\fwmu}{\hat{\boldsymbol \mu}}
\newcommand{\filmu}{\boldsymbol \mu}

\newcommand{\sindy}{SINDy\xspace}

\DeclareMathOperator*{\argmin}{arg\,min}

\definecolor{lightgray}{gray}{0.85}

\usepackage{newfloat}
\usepackage{listings}
\DeclareCaptionStyle{ruled}{labelfont=normalfont,labelsep=colon,strut=off} 
\floatstyle{ruled}
\newfloat{listing}{tb}{lst}{}
\floatname{listing}{Listing}
\title{Modeling Latent Non-Linear Dynamical System over Time Series}
\author{
    Ren Fujiwara,
    Yasuko Matsubara,
    Yasushi Sakurai
}
\affiliations{
    SANKEN, Osaka University\\

    r-fujiwr88@sanken.osaka-u.ac.jp, yasuko@sanken.osaka-u.ac.jp, yasushi@sanken.osaka-u.ac.jp
}

\usepackage{bibentry}

\begin{document}

\maketitle

\begin{abstract}
We study the problem of modeling a non-linear dynamical system when given a time series by deriving equations directly from the data. Despite the fact that time series data are given as input, models for dynamics and estimation algorithms that incorporate long-term temporal dependencies are largely absent from existing studies. 
In this paper, we introduce a latent state to allow time-dependent modeling and formulate this problem as a dynamics estimation problem in latent states.
We face multiple technical challenges, including (1) modeling latent non-linear dynamics and (2) solving circular dependencies caused by the presence of latent states.
To tackle these challenging problems, we propose a new method, \textbf{La}tent \textbf{No}n-\textbf{L}inear \textbf{e}quation \textbf{m}odeling (\method),
that can model a latent non-linear dynamical system and a novel alternating minimization algorithm for effectively estimating latent states and model parameters. In addition, we introduce criteria to control model complexity without human intervention. 
Compared with the state-of-the-art model, \method achieves competitive performance for estimating dynamics while outperforming other methods in prediction. 
\end{abstract}

\begin{links}
\link{Code}{https://github.com/renfujiwara/LaNoLem}
\end{links}

\section{Introduction}
\label{sec: 01_intro}
The past few years have seen a surge in methods for modeling dynamics as mathematical expressions from observed data, encompassing both advancements in traditional evolutionary algorithms and novel machine learning-based techniques \cite{sr4,sr5,sr3,sr2,sr1}. In particular, the sparse identification of non-linear dynamics (SINDy) framework \cite{sindy} has emerged as a leading approach for parsimonious modeling, and its variants have been proposed over a number of years \cite{ssr, sindy_constsr3, sindy_trsr3, mio-sindy}.

Despite these advances, limited research considers the temporal dependencies between data, even though many data are given as time series data. 
The sequential nature of the data can help us address various challenges. 
For example, distinguishing between noise and non-linearity is challenging when modeling data as non-linear dynamics from noisy data.
However, it is easily identifiable by employing a fitting algorithm based on the time dependencies because non-linearity is time-dependent while noise is independent at each time point. 
In addition, the sequential nature of the data helps stabilize future predictions. 
Initial states are critical for prediction in non-linear models, and unstable initial states reduce prediction accuracy. 
This can also be resolved by tracking state transitions in the model and estimating initial states based on them.
In time series analysis, such problems are treated as modeling dynamics in the latent state, i.e., modeling dynamics where noise in the data is removed or compressed to a dimension lower than that of the data.
Intuitively, we wish to solve the following problem:

\textit{
Given a time series, how can we estimate the latent states and dynamics with mathematical expressions?}

This problem presents us with the following two challenges: (1) how to formulate a model in latent states and (2) how to estimate model parameters and latent states simultaneously. 
A crucial insight into many dynamical systems of interest is that the function representing the transition consists of only a few but various types of terms \cite{sindy}. 
In other words, 
transitions among latent states must be represented by the non-linear dynamical system consisting of various terms of the latent states, 
and the estimation algorithm must incorporate a sparsity of these terms.
The presence of latent states also causes the following circular dependency in the algorithm: good latent states require well-estimated models, and the latent states must be properly estimated to find a suitable model. 
Non-linear state transitions and parameter sparsity further complicate this circular dependency. Specifically, non-linear state transitions introduce challenges in estimating latent states, while sparsity introduces challenges in estimating parameters.

In this paper, we tackle the above challenging problem and propose a method called \textbf{La}tent \textbf{No}n-\textbf{L}inear \textbf{e}quation \textbf{m}odeling (\method). 
\method can estimate the latent states and dynamics with mathematical expressions from observed data.
Our model has two components: (a) a non-linear dynamical system consisting of various terms of latent states and (b) criteria based on the minimum description length (MDL) principle to determine our model complexity.
To overcome the challenges caused by this model (i.e., non-linear state transitions and parameter sparsity), we develop a novel minimization algorithm that alternately estimates latent states and model parameters. In particular, our algorithm can ensure a fast convergence rate for estimating sparse parameters.

In summary, our contributions are as follows:
\begin{itemize}
    \item We propose a brief model representing latent non-linear dynamics and criteria for evaluating the trade-off between the accuracy and complexity of the model.
    \item We develop an efficient algorithm for estimating latent states and model parameters simultaneously, considering non-linear state transitions and sparsity in the parameters.
    \item Compared to state-of-the-art methods on 71 chaotic benchmark datasets, our model achieves competitive performance for estimating dynamics while consistently outperforming state-of-the-art in prediction tasks. 
\end{itemize}


\section{Proposed Model}
\label{sec: 02_model}
\mypara{Latent non-linear dynamical system}
In our model, we capture the latent dynamics of time series involving non-linear phenomena. 
We begin with the assumption that there are two classes of values:
\begin{description}
    \item[$\statevec(t)$\textnormal{:}] Latent states, i.e., $\sdim$-dimensional latent states at time point $t$, ($\statevec(t) = \{\statepoint_1(t),...\statepoint_\sdim(t)\}$).
    \item[$\estvec(t)$\textnormal{:}] Estimated values, i.e., $\datadim$-dimensional values that can be observed at time point $t$, ($\estvec(t) = \{\estpoint_1(t),...\estpoint_\datadim(t)\}$).
\end{description}
Here, we can only observe the estimated values $\estvec(t)$ at time point $t$. $\statevec(t)$ is a hidden vector that evolves over time as a dynamical system. Consequently, we consider dynamical systems described by the following equations:
\begin{align}
    \nonumber
    \statevec(t+1) &= \lMat\statevec(t) + \nlMat\nlstatevec + \sOf, \\
    \nonumber
    \estvec(t) &= \oMat\statevec(t)+\oOf.
\end{align}
$\lMat$ and $\nlMat$ describe the dynamics of latent states $\statevec(t)$, which capture linear and non-linear dynamics. $\oMat$ shows the observation projection, which generates the estimated value $\estvec(t)$ at each time point $t$. 
We also refer to $\sOf$ as state offsets and $\oOf$ as observation offsets. 
$\nlstatevec$ indicates non-linearities in the latent state space. In summary, we have the following:
\begin{definition}{(Full parameter set of \method: $\setparams$)} Let $\setparams$
be a complete set of \method parameters, namely, $\setparams = \params$.
\end{definition}
The non-linearities $\nlstatevec$ require adaption to various dynamics. Therefore, in our model, we describe $\nlstatevec$ as follows:
\begin{align}
\nonumber
\nlstatevec &= \nonlinearities.
\end{align}
Here, $\nldim$ represents the order of \method, and higher polynomials are denoted as $\ast 2$, $\ast 3$, ..., $\ast \nldim$, where $\ast 2$ denotes the quadratic non-linearities in the state $\statevec(t)=\{\statepoint_1(t),...,\statepoint_\sdim(t)\}]$:
\begin{align}
\nonumber
\statevec^{\ast 2}(t) &= [\statepoint_1^2(t), \statepoint_1(t)\statepoint_2(t), ..., \statepoint_\sdim^2(t)].
\end{align}
The non-linearities $\nlstatevec$ indicate that our model can have many terms and make the model more complex than necessary.
However, for many dynamical systems of interest, the function consists of only a few but various types of terms.
Our method represents this by making $\sparseparams$ sparse.
Consequently, our goal can be defined as follows.

\begin{problem}
Given a time series $\data \in \mathbb{R}^{\datalen \times \datadim}$, estimate latent states and parameter set, i,e.,
\begin{itemize}
    \item \textbf{estimate} latent states $\{\statevec(t)\}_{t=1}^\datalen$.
    \item \textbf{identify} full parameter set $\setparams$, s.t., $\sparseparams$ are sparse.
\end{itemize}
\end{problem}

\mypara{Criteria for controlling model complexity}
In the \method model, deciding model complexity (i.e., the order of \method $\nldim$) is an important problem in terms of accurate modeling. However, obtaining any error is also unreliable because it is necessary to consider the sparsity of the model, i.e., that $\lMat$ and $\nlMat$ are sparse matrices, to evaluate the true complexity of the proposed method.
We use the minimum description length (MDL) principle \cite{mdl} to deal with this problem. Although important, the MDL principle is a model selection criterion based on lossless compression principles, which does not directly address our problem. Thus, we define a new coding scheme for the summarization of a given model. 
We introduce an intuitive coding scheme, which is based on lossless compression principles. 
In short, the goodness of the model can be roughly described as $Cost_T = Cost(M)+Cost(X|M)$, where $Cost(M)$ shows the cost of describing the model $M$, and $Cost(X|M)$ represents the cost of describing the data $\data$ given the model $M$. 

Fig. \ref{fig: overview} shows an overview of our \method model for time series $\data$. 
We aim to estimate all these parameters, assuming that only a few important terms govern the dynamics. However, optimizing all the parameters is extremely expensive and depends on the noise.
When this is the case, how can we formulate the parameter estimation? We provide the answer in the next section.
\begin{figure}[t]
    \centering
    \includegraphics[width=0.96\linewidth]{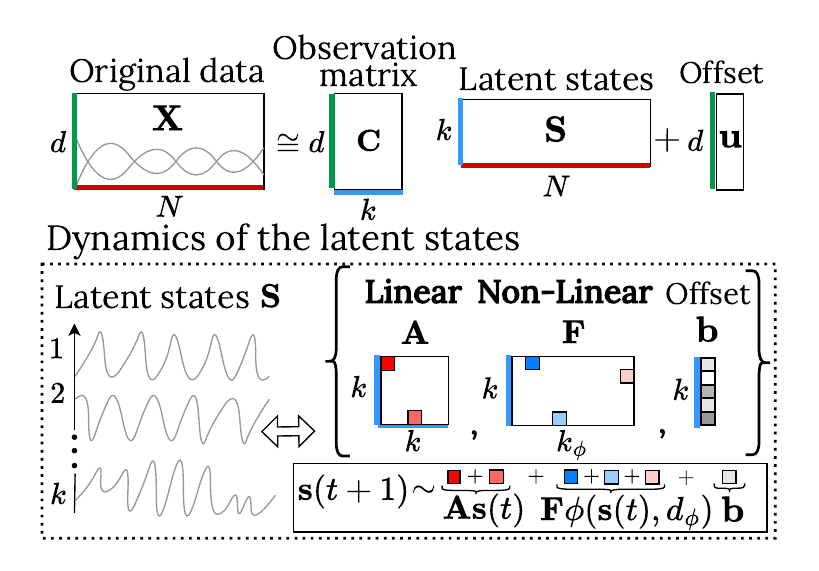}
    \vspace{-1.5em}
    \caption{
        Overview of \method: Given a data sequence $\data$, we extract latent states represented by a non-linear dynamical system, 
        where parameters for dynamics, i.e., $\lMat$ and $\nlMat$, are sparse.
    }
    \vspace{-1.5em}
    \label{fig: overview}
\end{figure}

\section{Optimization Algorithm}
\label{sec: 03_algorithm}
Given a data $\data$, we propose an algorithm to estimate:
\begin{itemize}
    \item the latent states $\hat{\statevec}(t) = \mathbb{E}[\statevec(t)], (t = 1, ..., \datalen)$;
    \item the governing parameter set $\setparams=\params$;
\end{itemize}
The goal of estimation is achieved by minimizing the negative likelihood of observed data. 
However, it is difficult to directly minimize the negative likelihood of incomplete data; we minimize the expected negative log-likelihood of the observation sequence.
Our overall optimization problem is
\begin{align}
    \label{eq:all_ob}
    \argmin_{\State, \setparams} -\nll(\data,\State,\setparams) + r(\lMat,\nlMat).
\end{align}
$\State$ is a set of latent states, i.e., $\State=\{\statevec(t)\}^\datalen_{t=1}$.
$r(\cdot)$ is the regularizer for penalizing non-zero solutions. 
In this paper, we use the elastic-net regularizer \cite{regularization} as $r(\cdot)$.
However, since the difference operation is expressed as $s(t + h) = s(t) + \Delta s(t)$ \cite{difference}, we must regularize that penalizes solutions that are not 1 in the diagonal component of $\lMat$. Therefore, we use the regularizer to differentiate the identity matrix I, i.e., $r(\mathbf{\Theta}) = \frac{1}{2}\lambda_2 ||\mathbf{\Theta} - \eye||^2_F + \lambda_1||\mathbf{\Theta} - \eye||_1$. 
$||\mathbf{\Theta}||_F$ represents the Frobenius norm of matrix $\mathbf{\Theta}$ and $||\mathbf{\Theta}||_1$ is the $l_1$ norm on every element of matrix $\mathbf{\Theta}$. 
Additionally, $\nll(\data,\State,\setparams)$ is as follows,
\begin{align}
    \nonumber
    \nll(\data, \State, \setparams) &= \mathbb{E}[-\sum^\datalen_{t=1}D(\datavec(t), \oMat\statevec(t)+\oOf, \oCov) - \frac{\datalen \log|\oCov|}{2}\\
    \nonumber
    &-\sum^{\datalen-1}_{t=1}D(\statevec(t+1), \lMat\statevec(t)+\nlMat\phi(\statevec(t),\nldim)+\sOf, \sCov) \\
    \label{eq:nll}
    &- \frac{(\datalen-1)\log|\sCov|}{2}],
\end{align}
\normalsize
where $D$ is the square of the Mahalanobis distance $D(x,y,\Sigma)=(x-y)^T\Sigma^{-1}(x-y)$. $\sCov$/$\oCov$ represents the covariance of Gaussian noises in states/observations.
\begin{algorithm}[t]
    \caption{\fitalgo ($\data$)}
    \label{alg: overview}
    \begin{algorithmic}[1]
    \STATE {\bf Input:} data $\data \in \mathbb{R}^{\datalen\times\datadim}$ \\
    \STATE {\bf Output:} (a) estimated latent states $\{\hat{\statevec}(t)\}_{t=1}^\datalen$ \\
    \ \ \ \ \ \ \ \ \ \ \ \ \ \ \
    (b) parameter set $\setparams=\params$
    \REPEAT
    \STATE /* \sfit*/ 
    \STATE Estimate latent space $\{\hat{\statevec}(t)\}_{t=1}^\datalen$ // Eq. \eqref{eq: fw_st}-\eqref{eq: bw_ed}
    \STATE Estimate the moments set $\mSet$ for \pfit \ // Eq. \eqref{eq: m_st}-\eqref{eq: m_ed}
    \STATE /* \pfit */
    \STATE $\setparams^{new}$ = \text{\pfit}($\{\hat{\statevec}(t)\}_{t=1}^\datalen,\mSet,\setparams) $
    \UNTIL{convergence;}
    \STATE \textbf{return} $\{\{\hat{\statevec}(t)\}_{t=1}^\datalen,\setparams\}$
    \end{algorithmic}
    \normalsize
\end{algorithm}

We must estimate latent states appropriately to find a suitable parameter set for $\data$. Simultaneously, a good latent state requires a well-estimated model. We escape this circular dependency by applying a novel alternating minimization.

Algorithm \ref{alg: overview} provides an overview of \fitalgo. 
Specifically, we propose an iterative method that alternately employs the following two sub-algorithms (\sfit and \pfit) until the cost function reaches a minimum value. Next, we describe each sub-algorithm in detail.
\subsection{\sfit}
The goal of \sfit is achieved by finding the marginal distributions for the latent variables conditional on the observation sequence. These inference problems can be solved efficiently using the sum-product message passing \cite{sum_product}.
In a linear setting, we can efficiently solve these inference problems with the Kalman filter in the context of LDS \cite{lds}. However, introducing transition models that depart from the linear Gaussian model leads to an intractable inference problem. 
Thus, we introduce one widely used approach to realize a Gaussian approximation by linearizing around the mean of the predicted distribution with a matrix of partial derivatives (the Jacobian), which gives rise to an extended Kalman filter \cite{ekf}. 
In \sfit, we use the Jacobian
$\jMat_{\statevec(t)}= \lMat + \frac{\partial \nlMat\nlstatevec}{\partial \statevec(t)}$.
This yields the following forward passing of the belief, as in LDS:
\begin{align}
    \label{eq: fw_st}
    \fwmu(t) &= \lMat\filmu(t-1)+\nlMat\boldsymbol \phi(\filmu(t-1), \nldim)+\sOf, \\
    \hat{\msCov}(t) &= \jMat_{\filmu(t-1)}\msCov(t-1)\jMat^T_{\filmu(t-1)}+\sCov, \\
    \mathbf{U}(t) &= \oMat \hat{\msCov}(t)\oMat^T+\oCov, \\
    \mathbf{K}(t) &= \hat{\msCov}(t)\oMat^T\mathbf{U}^{-1}(t),\\
    \filmu(t) &= \fwmu(t)+\mathbf{K}(t)(\datavec(t)-\oMat\fwmu(t)),\\
    \label{eq: fw_ed}
    \msCov(t) &= (\mathbf{I}-\mathbf{K}(t)\oMat)\hat{\msCov}(t).
\end{align}
We also obtain the backward passing equations:
\begin{align}
    \label{eq: bw_st}
    \hspace*{-1em} \msB(t) &= \msCov(t)\jMat^T_{\filmu(t)}\hat{\msCov}^{-1}(t+1), \\
    \hspace*{-1em} \hat{\statevec}(t) &= \filmu(t)+\msB(t)(\hat{\statevec}(t+1) - \fwmu(t+1)),\\
    \label{eq: bw_ed}
    \hspace*{-1em} \mathbf{W}(t) &= \msCov(t)+\msB(t)(\mathbf{W}(t+1)-\hat{\msCov}(t+1))\msB^T(t).
\end{align}
In this way, the optimal latent states $\{\hat{\statevec}(t)\}^\datalen_{t=1}$ are estimated. For \pfit, we require the following moments: 
\begin{align}
    \label{eq: m_st}
    \hspace*{-1em} \mathbb{E}[\statevec(t)] =& \hat{\statevec}(t), \\
    \hspace*{-1em} \mathbb{E}[\statevec(t)\statevec^T(t)] =& \mathbf{W}(t) + \hat{\statevec}(t)\hat{\statevec}^T(t), \\
    \hspace*{-1em} \mathbb{E}[\statevec(t+1)\statevec^T(t)] =& \mathbf{W}(t+1)\msB^T(t) + \hat{\statevec}(t+1)\hat{\statevec}^T(t).
\end{align}
To give our problem a \pfit-friendly form, we introduce a concatenated vector $\statevec_\phi(t) = [\statevec(t), \phi(\statevec(t), \nldim)]$
and concatenated matrix $\sparseMat =[\lMat, \nlMat]$.
We also require $\mathbb{E}[\statevec_\phi(t)]$. Here, we introduce the following moment generating function \cite{moment}.
\begin{definition}[Moment generating function:$M_X(i)$]
Given a vector $\mu$ and nonnegative definite matrix $\Sigma$,
the moment generating function of a random variable $X \sim \mathcal{N}(\mu, \Sigma)$ is a function $M_X(i)$ defined as
\begin{align}
    \nonumber
    M_X(i) &=\mathbb{E}[e^{i^TX}]
    (= e^{i^T\mu+\frac{1}{2}i^T\Sigma i}).
\end{align}
\end{definition}
\textit{Remark.} $M_X(i)$ is the exponential generating function of the moments of the probability distribution:
\begin{align}
    \label{eq: moment}
    \left. \mathbb{E}[X^n]=\frac{d^nM_X}{di^n} \right|_{i=0}.
\end{align}
Since $\statevec_\phi(t)$ is a higher polynomial of $\statevec(t)$, the moment $\mathbb{E}[\statevec_\phi(t)]$ can be regarded as a higher-order moment of a multivariate Gaussian distribution, and we can generate these values with Eq. \eqref{eq: moment}. 

A large state noise changes the system, i.e., a different system incorporating the noise should be estimated if the state noise is large. Therefore, we can assume that the high-order state noise is sufficiently small without loss of generality.
We use the following approximated moments:
\begin{align}
    \mathbb{E}[\statevec_\phi(t)\statevec_\phi^T(t)] &= \mathbb{E}[\statevec_\phi(t)]\mathbb{E}[\statevec_\phi(t)]^T, \\
    \label{eq: m_ed}
    \mathbb{E}[\statevec(t+1)\statevec_\phi^T(t)] &= \mathbb{E}[\statevec(t+1)]\mathbb{E}[\statevec_\phi(t)]^T.
\end{align}
\subsection{\pfit}
Once we have the estimated latent states, we can run the algorithm to determine the parameter set $\setparams$. 
The model dynamics parameters $\lMat$ and $\nlMat$ should include only a few important terms, i.e., they should be sparse matrices. However, simultaneously estimating these sparse matrices along with other parameters presents a considerable challenge.
Therefore, we split parameter set $\setparams$ into two subsets, i.e., $\setsparse = \sparseparams$ and $\setnonsparse = {\setparams \setminus \setsparse} = \nonsparseparams$, each of which corresponds to a sparse/non-sparse parameter set, and try to fit the parameter sets separately. 
Algorithm \ref{alg: pfit} shows the \pfit optimization process. In \pfit, a set of moments (Eq. \eqref{eq: m_st} - Eq. \eqref{eq: m_ed}) is defined as a moment set $\mSet$ as follows, which is used to estimate model parameters.
\begin{definition}[Moment set: $\mSet$]
    \par
    Let $\mSet$ be a set of moments, 
    \begin{align}
    \nonumber
    \mSet =& \{\mathbb{E}[\statevec(t)], \mathbb{E}[\statevec(t)\statevec^T(t)], \mathbb{E}[\statevec(t+1)\statevec^T(t)], \\
    & \mathbb{E}[\statevec_\phi(t)], \mathbb{E}[\statevec_\phi(t)\statevec_\phi^T(t)], \mathbb{E}[\statevec(t+1)\statevec_\phi^T(t)]\}^\datalen_{t=1}.
    \end{align}
\end{definition}
\mypara{Fitting sparse parameter set}
In each iteration, we need to minimize Eq. \eqref{eq:all_ob} with respect to $\sparseMat$, which is equivalent to minimizing the following function $f(\sparseMat)$:
\begin{align}
    \label{eq: s_obs}
    \hspace{-0.5em}
    \nonumber
    f(\sparseMat) = \mathbb{E}[\sum^{\datalen-1}_{t=1}D(\statevec(t+1), \sparseMat\statevec_\phi(t)+\sOf, \sCov)] \\ + \frac{1}{2}\lambda_2||\sparseMat -\eye||^2_F + \lambda_1||\sparseMat-\eye||_1.
\end{align}
As we can see, $f(\sparseMat)$ is convex but non-differentiable, and we can easily decompose $f(\sparseMat)$ into two parts: $f(\sparseMat) = g(\sparseMat) + h(\sparseMat)$, as shown in Eq. \eqref{eq: s_obs}. Since $g(\sparseMat)(=\mathbb{E}[\sum^{\datalen-1}_{t=1}D(\statevec(t+1), \sparseMat\statevec_\phi(t)+\sOf, \sCov)] + \frac{1}{2}\lambda_2||\sparseMat-\eye||^2_F)$ is differentiable, we can adopt the generalized gradient descent algorithm to minimize $f(\sparseMat)$. The update rule is: $\sparseMat(i+1) = \text{prox}_{\alpha\lambda}(\sparseMat(i)-\alpha  \nabla g(\sparseMat(i)) - \eye)$ where $\alpha$ is the step size at iteration $i$ and the proximal function $\text{prox}_{\alpha\lambda} (\sparseMat)$ is defined as the soft-threshold proximal operator $Th_{\alpha\lambda}(\beta)$, which has the following closed-form solution:
\begin{align}
    \nonumber
    Th_{\alpha\lambda_1}(\beta)&=\begin{cases}\beta-\alpha\lambda_1\ & \text{if}\; \beta<\alpha\lambda_1, \\0& \text{if}\; -\alpha\lambda_1 \leq \beta \leq \alpha\lambda_1, \\ \beta+\alpha\lambda_1& \text{if}\;\beta< - \alpha\lambda_1.\end{cases}
\end{align}
Also, the gradient $\nabla g(\sparseMat)$ is as follows:
\begin{align}
    \nonumber
    \nabla g(\sparseMat)=&\sCov^{-1}(\sum^{\datalen-1}_{t=1}
    \sparseMat\mathbb{E}[\statevec_\phi(t)\statevec_\phi^T(t)] + \sOf\mathbb{E}[\statevec_\phi^T(t)]\\
    &
    \label{eq: grad}
    -\mathbb{E}[\statevec(t+1)\statevec_\phi^T(t)]) + \lambda_2(\sparseMat-\eye).
\end{align}
There remains the question of how to determine the step size $\alpha$. The appropriate step size is important in the gradient method. However, it is difficult to determine the step size in advance e.g., by a grid search, since each \sfit updates the latent states. The appropriate step size in a linear setting has been discussed in \cite{rlds}, and we can extend this to our case.
Proposition \ref{prop: stepsize} allows us to select the step size while assuring its fast convergence rate.
\begin{proposition}[]
\label{prop: stepsize}
Generalized gradient descent with a fixed step size $\alpha \leq 1/(\|\sCov^{-1}\|_F \cdot \| \sum^{\datalen-1}_{t=1}
    \mathbb{E}[\statevec_\phi(t)\statevec_\phi^T(t)]\|_F + \lambda_2)$ for minimizing has a convergence rate $O(1/i)$, where $i$ is the number of iterations.
\end{proposition}
\begin{proof}
Please see Appendix C.
\end{proof}

\mypara{Fitting non-sparse parameter set}
To estimate the non-sparse parameter, using the derivatives of \eqref{eq:all_ob} with respect to each of the components of $\setnonsparse$ and setting them at zero yields (i.e., $\frac{d\nll}{d\setnonsparse}=0$), and here we omit the details provided in Appendix D.
\begin{algorithm}[t]
    \caption{\pfit ($\{\hat{\statevec}(t)\}_{t=1}^\datalen, \mSet, \setparams$)}
    \label{alg: pfit}
    \begin{algorithmic}[1]
    \STATE {\bf Input:} (a) estimated latent states $\{\hat{\statevec}(t)\}_{t=1}^\datalen$ \\
    \ \ \ \ \ \ \ \ \ \ \ \
    (b) moments set
    $\mSet$ \\
    \ \ \ \ \ \ \ \ \ \ \ \
    (c) parameter set
    $\setparams = \{\setsparse, \setnonsparse\}$\\
    \STATE {\bf Output:} new parameter set $\setparams^{new} = \{\setsparse^{new}, \setnonsparse^{new}\}$
    \STATE $\sparseMat = [\lMat,\nlMat]$
    \STATE /* Compute the fixed stepsize $\alpha$ */
    \STATE $\alpha=1/(\|\sCov^{-1}\|_F \cdot \| \sum^{\datalen-1}_{t=1}
    \mathbb{E}[\statevec_\phi(t)\statevec_\phi^T(t)]\|_F + \lambda_2)$
    \REPEAT
    \STATE Compute gradient $\nabla g(\sparseMat)$ // Eq. \eqref{eq: grad}
    \STATE $\sparseMat = Th_{\alpha\lambda}(\sparseMat-\alpha \nabla g(\sparseMat))$
    \UNTIL{convergence}
    \STATE $\setsparse^{new} \gets \sparseMat$
    \STATE /* Fit other parameters (See details in Appendix D)*/
    \STATE $\setnonsparse^{new} = \argmin_{\setnonsparse} \nll(\data, \setsparse^{new}, \setnonsparse)$
    \STATE \textbf{return} $\{\setsparse^{new}, \setnonsparse^{new}\}$
    \end{algorithmic}
    \normalsize
\end{algorithm}

\section{Experimental Results}
\label{sec: 04_exp}
\begin{figure*}[t]
    \includegraphics[width=\linewidth]{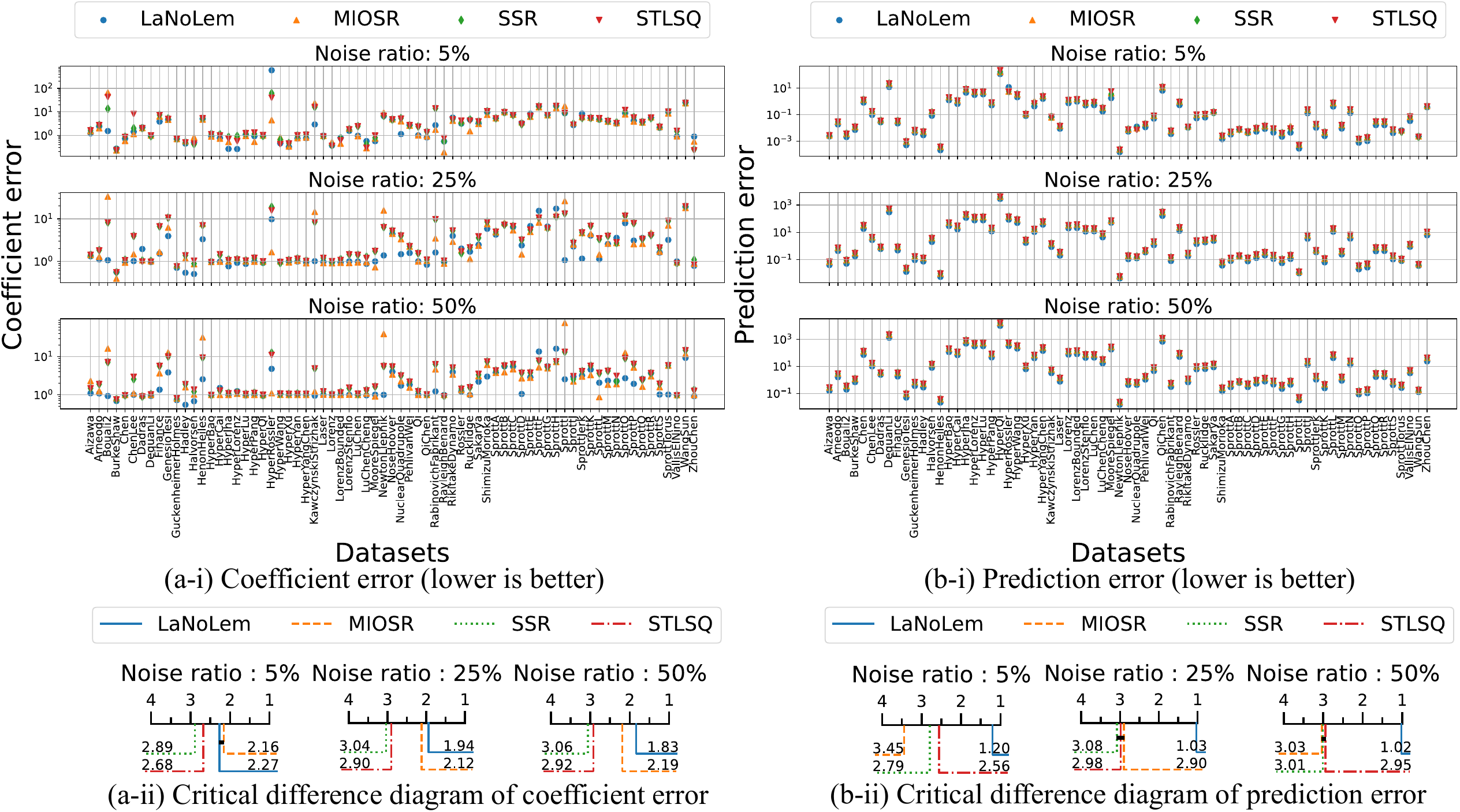}
    \vspace{-1.5em}
    \caption{Accuracy and robustness for dysts dataset: (a) \method achieves competitive performance for coefficient error.
    (b) \method also consistently achieved the lowest prediction error.}
    \label{fig: exp_syn}
    \vspace{-1.0em}
\end{figure*}
In this section, we run experiments on synthetic data where there are ground truth systems to evaluate the accuracy and robustness of our method.
We compare \method to three state-of-the-art baselines to measure accuracy and robustness, and we show the importance of latent states for them.
\subsection{Experimental Setup}
\mypara{Datasets} We use synthetic data obtained from dysts database \cite{dysts}, which provides data, equations, and dynamical properties for chaotic systems exhibiting strange attractors and coming from disparate scientific fields. In our experiments, we consider 71 systems representing ordinary differential equations (ODEs) with polynomial non-linearities that have no more than four degrees. 
Most importantly, this allows us to use the true governing equations to evaluate the model performance.

\mypara{Evaluation Metrics}
We use normalized coefficient error\footnote{$\text{Coefficient error} = \frac{\|\text{True coefficients}-\text{Learned coefficients}\|_2}{\|\text{True coefficients}\|_2}$}, 
which simply measures the normalized Euclidean distance between the true and the learned coefficients.
The coefficients are parameters for representing dynamics (e.g., $\lMat$ and $\nlMat$ are the coefficients in our method). The lower the coefficient error, the higher the modeling accuracy. We also use mean squared error (MSE) as the prediction error.

\mypara{Baseline Methods}
We compared our approach with the following methods: MIOSR \cite{mio-sindy}, SSR \cite{ssr}, and STLSQ \cite{sindy}, which are optimization algorithms in the SINDy framework.

The experimental settings are detailed in Appendix E, which contains a detailed description of the experimental conditions and hyperparameters used in our study.

\subsection{Accuracy and Robustness}
We discuss the accuracy of \method in estimating the non-linear dynamics of data and prediction.
In this experiment, we answer the following questions: How accurately and robustly does our method estimate a system and generate future values from noisy data, namely detect the appropriate coefficients (i.e., $\lMat$ and $\nlMat$) of the governing equation and predict one step ahead?
We vary the noise ratio\footnote{This represents the ratio of additive Gaussian noise scaled to a certain percentage of the $\ell_2$ norm of each dataset, i.e., $\text{noise ratio (\%)}=\frac{\text{$\ell_2$ norm of noise}}{\text{$\ell_2$ norm of data}}*100$.}, 
each with additive Gaussian noise, when varying random seeds in each experiment. We evaluate the quality of \method in estimating the coefficients and prediction when we set the noise ratio at 5\%, 25\%, and 50\%. 
Fig. \ref{fig: exp_syn} shows the summary of the average errors and the critical difference diagram for estimating and predicting each noise ratio. The critical difference diagrams are based on the Wilcoxon-Holm method \cite{cd_diagram}, where methods that are not connected by a bold line are significantly different as regards average rank.
\begin{table}[t]
\hspace{-1.0em}
\vspace{-1.0em}
\centering
\normalsize
\caption{$1^{\text{st}}$ Count in dysts dataset (higher is better).}
\label{table: win_count}
\vspace{0.5em}
\scalebox{0.9}
{
\begin{tabular}{lcccccc}
\toprule
\large \textbf{Measures} & \multicolumn{3}{c}{\large \textbf{Coefficient error}} & \multicolumn{3}{c}{\textbf{Prediction error}}\\
\cmidrule(lr){1-1} \cmidrule(lr){2-4} \cmidrule(lr){5-7}
\large \textbf{Noise ratio} & 5\% & 25\% & 50\%  & 5\% & 25\%& 50\%\\
\midrule
\large \textbf{STLSQ} & 6 & 4 & 3 &\underline{4} & 0 & 0\\
\large \textbf{SSR} & 6 & 5 & 0 & 0 & 0 & 0\\
\large \textbf{MIOSR} & \textbf{33} & \underline{25} & \underline{26} & 1 & 0  & 0\\
\midrule
\large \textbf{\method} & \underline{26} & \textbf{37} & \textbf{42} & \textbf{66} & \textbf{71} & \textbf{71}\\
\bottomrule
\end{tabular}
}
\vspace{-1.0em}
\end{table}
Table \ref{table: win_count} also shows the win count in the 71 datasets.
The results presented in the figures and tables demonstrate the effectiveness of \method. 
At a low noise ratio (5\%), MIOSR is competitive with the \method in coefficient error. Fig. \ref{fig: exp_syn} (a-ii) also demonstrate that \method and MIOSR aren't significantly different because they are connected by a bold line at the low noise.
However, as the noise level increases, the accuracy of the baselines decreases significantly, while the accuracy of \method decreases less than with the other methods.
Our method also consistently achieves the best prediction error.
\method achieved low coefficient error because it can filter out background noise when estimating latent states and systems using the temporal dependency of the data.
On the other hand, other methods treat time series data as split samples and cannot successfully separate noise and non-linearity. As a result, they overfit the noise and reduce the accuracy. Furthermore, because the introduction of latent states provides appropriate initial conditions, \method also provides stable prediction, which contributes significantly to consistent and highly accurate prediction.
We also show the numerical results, including statistical information (i.e., means and standard deviations), in Appendix G.

\mypara{Ablation study} In Appendix F, we prepare a limited version of the proposed method and compare it with the original \method for a more detailed evaluation.

\section{Case Studies}
\label{sec:05_app}
In this section, we describe the performance of \method using synthetic and real datasets. 
This section is designed to answer the following questions:
\begin{description}
    \item[\textit{Q1. Effectiveness}:] What are the advantages of introducing non-linear equations in latent states?
    \item[\textit{Q2. Practicality}:] How well does \method obtain meaningful insights from real-world datasets?
\end{description}
Through these questions, we show how latent states can be useful in various problems. Note that existing methods without latent states, such as the SINDy framework, can't be applied directly to scenarios in these questions.

\subsubsection{\textit{Q1. Effectiveness}}
Here, we discuss the effectiveness of \method in interpolating missing values, a typical issue addressed in dynamical systems in latent states \cite{dcmf, dynammo}. We show the advantages of describing the latent state with non-linear equations by comparison with rLDS \cite{rlds}, a linear dynamic system that introduces regularization. 
Fig. \ref{fig: q1_preview} shows results for the Halvorsen attractor \cite{halvorsen} when we set the noise ratio at 50\%. Fig. \ref{fig: q1_preview} (i) shows the results for missing value interpolation using the \method and rLDS for noisy data.
The intervals to be interpolated are each represented by a solid colored line, and the given input, except for the complementary interval, is represented by a gray line. Also, the solid orange line is noise-free (i.e., ground truth). \method (solid green line) completes a trajectory similar to the ground truth, whereas rLDS shows a completely different trajectory (solid red line). rLDS and \method introduce latent states, but rLDS is incapable of handling non-linear dynamics.

Our method can also effectively estimate dynamics even when the given data contains missing values. Each row of the heat maps in Fig. \ref{fig: q1_preview} (ii) and (iii) corresponds to $\frac{dx}{dt}$, $\frac{dy}{dt}$, and $\frac{dz}{dt}$. 
As shown in Fig. \ref{fig: q1_preview} (ii), the dynamics of Halvorsen have linear terms $\{x, y, z\}$ and non-linear terms $\{x^2, y^2, z^2\}$.
For example, the first row of the heatmap corresponds to $\frac{dx}{dt}$. This heatmap shows that $\frac{dx}{dt}$ contains the terms $\{x,y,z,y^2\}$. The same is true for the other rows.
All these results show that \method can estimate accurate and meaningful dynamics.
\begin{figure}[t]
    \centering
    \includegraphics[width=\linewidth]{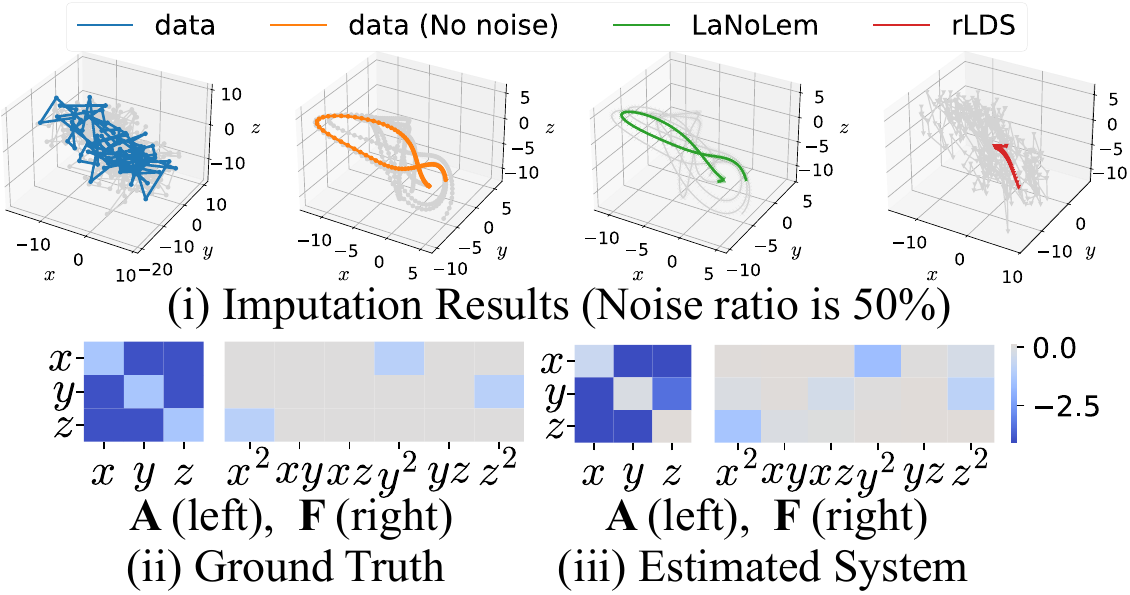} 
    \vspace{-2.0em}
    \caption{Effectiveness of \method: \method can accurately estimate the system and interpolate missing trajectories even when those are noisy and missing data.}
    \label{fig: q1_preview}
    \centering
    \includegraphics[width=\linewidth]{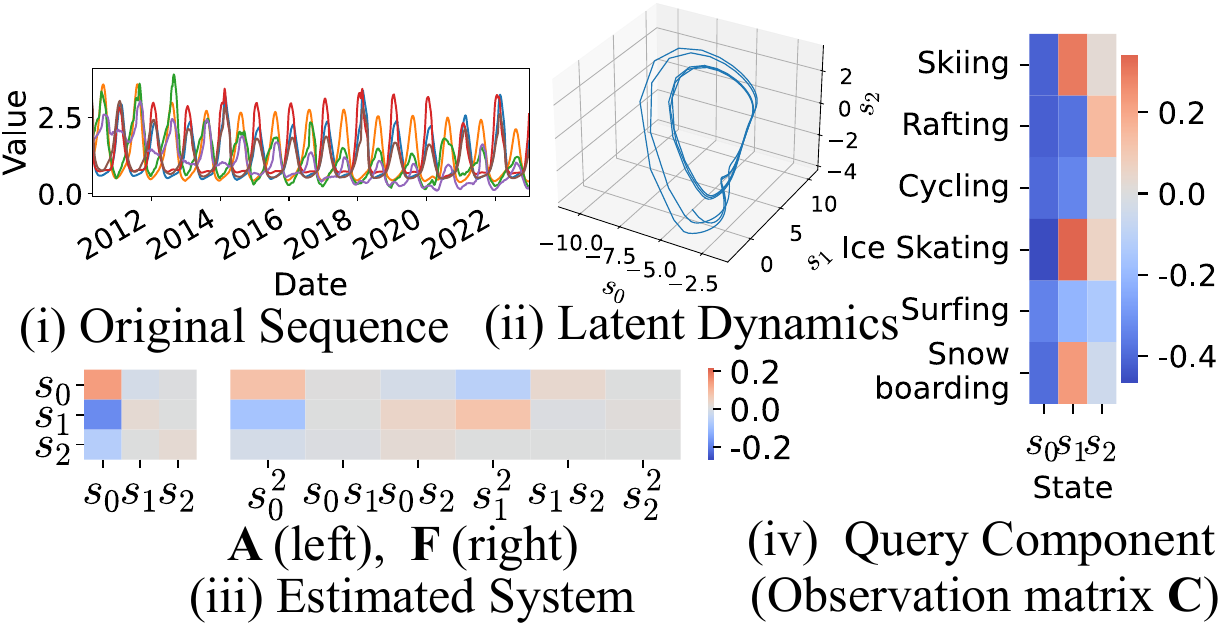}
    \vspace{-2.0em}
    \caption{Practicality of \method: \method can estimate (ii) latent dynamics, (iii) the system, and (iv) query components in (i) Google Trends data, which consist of the search volumes for six outdoor-related keywords.}
    \label{fig: q2_preview}
    \vspace{-1.5em}
\end{figure}

\subsubsection{\textit{Q2. Practicality}}
In this section, we present each example using the following real dataset. 
\begin{description}
    \item[Web activity data:] The dataset contains web-search counts related to outdoor query sets collected over twelve years from 2010 to 2022. 
    \vspace{-1.0em}
\end{description}
Here, we show what features \method can extract from real-world data sets.
\method allows flexible modeling by introducing latent states. As a result, our method separates noise and non-linearity well. In addition,
our method can estimate the dynamics in low-dimensional latent states (i.e., $\sdim < \datadim$) that existing studies cannot estimate. 
Modeling in low-dimensional latent states is important because it can be applied to various time series analyses (e.g., forecasting \cite{relinet, orbitmap}).
Fig. \ref{fig: q2_preview} shows our results regarding search volume data. 
Fig. \ref{fig: q2_preview} (i) shows the original data. \method can discover (ii) the latent states and (iii) the system, with which it factorizes the data into (iv) query factors, where six queries are represented as three latent states. 
We can interpret the circumstances of web activities from the connection between the latent states and query factors. 
For example, 
skiing, ice skating, and snowboarding are each assigned a positive weight with respect to state $s_1$, while cycling, surfing, and rafting are each assigned a negative weight.
This agrees with our intuition: skiing, ice skating, and snowboarding are winter activities, while the others are summer activities. Furthermore, Fig \ref{fig: q2_preview} (ii)  shows the estimated latent state. This is generated from the system in Fig \ref{fig: q2_preview} (iii), and the trajectory seems to have a limit cycle. This indicates the latent periodicity of each activity.

\section{Scope and Limitation}
\label{sec: 06_sl}
\mypara{Class of mathematical expression}
Our method deals with polynomials, which is a narrower class of formulas than related works. 
However, the number of non-linear models that can be expressed in polynomial form is much larger than one might imagine. For example, according to \cite{chaos_bench}, there are 71 non-linear models that polynomials can describe.
Also, replacing our model would make it possible to deal with a broader class of formulas.
However, it is computationally expensive and requires a bolder approximation for the expected value.

\mypara{Initialization} 
Our algorithm does not guarantee convergence to the global optimum, so it is sensitive to initial values. 
For example, in real datasets, our method is initialized using singular value decomposition (SVD). Setting initial values to improve the accuracy of our method is an interesting problem, but it is outside the scope of this paper.

\section{Related Work}
\label{sec: 07_related}
In this section, we briefly describe investigations related to this research.
Table \ref{table: capability} shows the relative advantages of our method, and only \method meets all the requirements. We separate the details of previous studies into two categories: modeling non-linear dynamics and time series analysis.

\mypara{Modeling non-linear dynamics} 
Estimating mathematical expressions from data represents a significant challenge in a wide range of diverse areas of science and engineering. 
A seminal work, \sindy, has been published on sparse regression methods for system identification \cite{sindy}. 
In that study, the model was estimated using sequentially thresholded least squares (STLSQ). More recently, representative algorithms for the \sindy framework, namely stepwise sparse regression (SSR) \cite{ssr} and mixed-integer optimized sparse regression (MIOSR) \cite{mio-sindy} have been proposed. 
These methods are effective because they provide interpretable models (systems) for various types of data. However, none of these methods are intended for estimating latent states and dynamics; therefore, they have lower model estimation for noisy data. 
Moreover, extracting interpretable components from the data and representing their dynamics is impossible. 
Symbolic regression (SR) can identify the tractable formula, i.e., $y=f(x)$, that best fits a given data pair $(x,y)$ \cite{sr1,sr2,sr3,sr4,sr5}. 
These methods can accurately estimate the non-linear relationship between given variables.
However, these methods cannot extract latent states and estimate their latent dynamics from multiple time series. Moreover, they need target variables (i.e., derivatives) when modeling dynamics such as ODEs. Unlike our methods and SINDy frameworks, 
these methods must address noise amplification when calculating the derivative value from the noisy data.
Therefore, these methods cannot be applied directly to the problem of estimating models from time series, as discussed in this study.

\mypara{Time series analysis}
Linear dynamical systems (LDS) \cite{lds} utilize latent space to capture temporal dependencies and learn latent dynamics. They have been applied to various analytical tasks, including forecasting, clustering, and pattern mining \cite{dynammo,plif,dcmf,ssm,netdyna,lds-cl}. Switching linear dynamical systems (SLDS) \cite{slds, slds2} enhance this framework by incorporating both discrete states that represent different motion modes and continuous states that describe the dynamics of each mode. This allows the representation of more complex time series. 
The regularized LDS model proposed in \cite{rlds} aims to learn an LDS model with a low-rank transition matrix from a limited number of sequences. Compared with regular LDS models, regularized LDS is able to find the essential dimensions of hidden states and prevent overfitting problems in advance.
While effective, these approaches typically assume state transitions governed by linear equations.
Estimating non-linear dynamical systems has also been attracting the attention of many researchers \cite{nlds_em,nlds_opt,nlds-opt2,nlds-opt3,nlds-opt4}. However, none of these non-linear models focus on modeling non-linear dynamics as mathematical expressions.






\newcommand*\rot{\rotatebox{90}}

\newcommand*\OK{\ding{51}}
\newcommand*\NO{-}
\newcommand{\SOME}{some}

\begin{table}[t]
\centering
\caption{
Capabilities of approaches.
}
\vspace{-0.5em}
\label{table: capability}
\scalebox{0.75}{
\hspace{-1.0em}
\begin{tabular}{l|c|c|c|c}
\toprule
&
LDS/++ &
SR/++&
\sindy/++ &
\textbf{\method} \\

\midrule
\rowcolor{lightgray}
Non-linear equation &\NO &\OK &\OK &\OK\\
Sparse Term &\NO &\SOME &\OK &\OK\\
\rowcolor{lightgray}
Robustness to differentiation &\NO &\NO &\OK &\OK\\
Modeling latent dynamics  &\OK &\NO &\NO &\OK\\

%

%
%

%
\bottomrule
\end{tabular}
}
\normalsize
\vspace{-1.0em}
\end{table} 

\section{Conclusion}
\label{sec: 08_conclusion}
This paper presented \method, an intuitive method for estimating both latent states and non-linear dynamics. 
Our main idea is to introduce latent states to allow the estimation of dynamics that consider the temporal dependencies of the data.
The important point is that the latent non-linear dynamical system and various dynamical systems should consist of fewer terms than the space of possible functions. 
Therefore, we propose a non-linear dynamical system consisting of latent states and various non-linear terms of states. We also propose new criteria to control its complexity. Finally, we develop a new alternating minimization algorithm to estimate our model effectively.
We also demonstrated the practicality and effectiveness of \method on various types of time series.
We can highlight the following key results:
\begin{itemize}
    \item We have introduced an effective model that represents latent non-linear dynamics and criteria for evaluating the trade-offs between model accuracy and complexity.
    \item We presented an efficient algorithm capable of simultaneously estimating latent states and model parameters, considering non-linear state transitions and 
    sparsity in the parameters.
    \item When benchmarked against state-of-the-art baselines, our model demonstrates competitive performance in estimating dynamics and maintains the leading performance in prediction tasks.
\end{itemize}

\section*{Acknowledgements}
This work was partly supported by
JST CREST JPMJCR23M3,
JSPS KAKENHI Grant-in-Aid for Scientific Research Number
JP24KJ1618.    

\bibliography{
BIB/icml24,
BIB/si,
BIB/time_series,
BIB/data}

\begin{thebibliography}{42}
\providecommand{\natexlab}[1]{#1}

\bibitem[{Akaike(1974)}]{aic}
Akaike, H. 1974.
\newblock A new look at the statistical model identification.
\newblock \emph{IEEE Transactions on Automatic Control}, 19(6): 716--723.

\bibitem[{Baier, Aspandi, and Staab(2023)}]{relinet}
Baier, A.; Aspandi, D.; and Staab, S. 2023.
\newblock ReLiNet: Stable and Explainable Multistep Prediction with Recurrent Linear Parameter Varying Networks.
\newblock In \emph{Proceedings of the Thirty-Second International Joint Conference on Artificial Intelligence}, 3461--3469.

\bibitem[{Bertsimas and Gurnee(2023)}]{mio-sindy}
Bertsimas, D.; and Gurnee, W. 2023.
\newblock Learning sparse nonlinear dynamics via mixed-integer optimization.
\newblock \emph{Nonlinear Dynamics}, 111(7): 6585--6604.

\bibitem[{Boninsegna, Nüske, and Clementi(2018)}]{ssr}
Boninsegna, L.; Nüske, F.; and Clementi, C. 2018.
\newblock {Sparse learning of stochastic dynamical equations}.
\newblock \emph{The Journal of Chemical Physics}, 148(24): 241723.

\bibitem[{Brunton, Proctor, and Kutz(2016)}]{sindy}
Brunton, S.~L.; Proctor, J.~L.; and Kutz, J.~N. 2016.
\newblock Discovering governing equations from data by sparse identification of nonlinear dynamical systems.
\newblock \emph{Proceedings of the national academy of sciences}, 113(15): 3932--3937.

\bibitem[{Burlacu, Kronberger, and Kommenda(2020)}]{sr3}
Burlacu, B.; Kronberger, G.; and Kommenda, M. 2020.
\newblock Operon C++: An Efficient Genetic Programming Framework for Symbolic Regression.
\newblock In \emph{Proceedings of the 2020 Genetic and Evolutionary Computation Conference Companion}, GECCO '20, 1562–1570.

\bibitem[{Cai et~al.(2015)Cai, Tong, Fan, and Ji}]{dcmf}
Cai, Y.; Tong, H.; Fan, W.; and Ji, P. 2015.
\newblock Fast Mining of a Network of Coevolving Time Series.
\newblock In \emph{Proceedings of the 2015 {SIAM} International Conference on Data Mining}, 298--306. {SIAM}.

\bibitem[{Champion et~al.(2020)Champion, Zheng, Aravkin, Brunton, and Kutz}]{sindy_constsr3}
Champion, K.; Zheng, P.; Aravkin, A.~Y.; Brunton, S.~L.; and Kutz, J.~N. 2020.
\newblock A unified sparse optimization framework to learn parsimonious physics-informed models from data.
\newblock \emph{IEEE Access}, 8: 169259--169271.

\bibitem[{Chen and Poor(2022)}]{lds-cl}
Chen, Y.; and Poor, H.~V. 2022.
\newblock Learning mixtures of linear dynamical systems.
\newblock In \emph{International Conference on Machine Learning}, 3507--3557. PMLR.

\bibitem[{Dabrowski et~al.(2018)Dabrowski, Rahman, George, Arnold, and McCulloch}]{ssm}
Dabrowski, J.~J.; Rahman, A.; George, A.; Arnold, S.; and McCulloch, J. 2018.
\newblock State Space Models for Forecasting Water Quality Variables: An Application in Aquaculture Prawn Farming.
\newblock In \emph{Proceedings of the 24th ACM SIGKDD International Conference on Knowledge Discovery {\&} Data Mining}, 177–185. {ACM}.

\bibitem[{Feller(1991)}]{moment}
Feller, W. 1991.
\newblock \emph{An introduction to probability theory and its applications, Volume 2}, volume~81.
\newblock John Wiley \& Sons.

\bibitem[{Fornasier and Rauhut(2008)}]{grad2}
Fornasier, M.; and Rauhut, H. 2008.
\newblock Iterative thresholding algorithms.
\newblock \emph{Applied and Computational Harmonic Analysis}, 25: 187--208.

\bibitem[{Foster, Sarkar, and Rakhlin(2020)}]{nlds-opt3}
Foster, D.; Sarkar, T.; and Rakhlin, A. 2020.
\newblock Learning nonlinear dynamical systems from a single trajectory.
\newblock In \emph{Proceedings of the 2nd Conference on Learning for Dynamics and Control}, volume 120, 851--861.

\bibitem[{Fox et~al.(2008)Fox, Sudderth, Jordan, and Willsky}]{slds2}
Fox, E.~B.; Sudderth, E.~B.; Jordan, M.~I.; and Willsky, A.~S. 2008.
\newblock Nonparametric Bayesian Learning of Switching Linear Dynamical Systems.
\newblock In \emph{Advances in Neural Information Processing Systems 21, Proceedings of the Twenty-Second Annual Conference on Neural Information Processing Systems}, 457--464.

\bibitem[{Ghahramani and Roweis(1998)}]{nlds_em}
Ghahramani, Z.; and Roweis, S. 1998.
\newblock Learning nonlinear dynamical systems using an EM algorithm.
\newblock \emph{Advances in neural information processing systems}, 11.

\bibitem[{Gilpin(2021{\natexlab{a}})}]{dysts}
Gilpin, W. 2021{\natexlab{a}}.
\newblock Chaos as an interpretable benchmark for forecasting and data-driven modelling.
\newblock In Vanschoren, J.; and Yeung, S., eds., \emph{Proceedings of the Neural Information Processing Systems Track on Datasets and Benchmarks}.

\bibitem[{Gilpin(2021{\natexlab{b}})}]{chaos_bench}
Gilpin, W. 2021{\natexlab{b}}.
\newblock Chaos as an interpretable benchmark for forecasting and data-driven modelling.
\newblock In Vanschoren, J.; and Yeung, S., eds., \emph{Proceedings of the Neural Information Processing Systems Track on Datasets and Benchmarks}, volume~1.

\bibitem[{Grünwald(2007)}]{mdl}
Grünwald, P.~D. 2007.
\newblock \emph{{The Minimum Description Length Principle}}.
\newblock The MIT Press.

\bibitem[{Hairi, Tong, and Ying(2019)}]{netdyna}
Hairi; Tong, H.; and Ying, L. 2019.
\newblock NetDyna: Mining Networked Coevolving Time Series with Missing Values.
\newblock In \emph{2019 {IEEE} International Conference on Big Data {(IEEE} BigData)}, 503--512. {IEEE}.

\bibitem[{Ismail~Fawaz et~al.(2019)Ismail~Fawaz, Forestier, Weber, Idoumghar, and Muller}]{cd_diagram}
Ismail~Fawaz, H.; Forestier, G.; Weber, J.; Idoumghar, L.; and Muller, P.-A. 2019.
\newblock Deep learning for time series classification: a review.
\newblock \emph{Data mining and knowledge discovery}, 33(4): 917--963.

\bibitem[{Kakade et~al.(2011)Kakade, Kanade, Shamir, and Kalai}]{nlds-opt4}
Kakade, S.~M.; Kanade, V.; Shamir, O.; and Kalai, A. 2011.
\newblock Efficient Learning of Generalized Linear and Single Index Models with Isotonic Regression.
\newblock In \emph{Advances in Neural Information Processing Systems}, volume~24.

\bibitem[{Kalman(1963)}]{lds}
Kalman, R.~E. 1963.
\newblock Mathematical description of linear dynamical systems.
\newblock \emph{Journal of the Society for Industrial and Applied Mathematics, Series A: Control}, 1(2): 152--192.

\bibitem[{Kaptanoglu et~al.(2021)Kaptanoglu, Callaham, Aravkin, Hansen, and Brunton}]{sindy_trsr3}
Kaptanoglu, A.~A.; Callaham, J.~L.; Aravkin, A.; Hansen, C.~J.; and Brunton, S.~L. 2021.
\newblock Promoting global stability in data-driven models of quadratic nonlinear dynamics.
\newblock \emph{Physical Review Fluids}, 6(9): 094401.

\bibitem[{Kaptanoglu et~al.(2022)Kaptanoglu, de~Silva, Fasel, Kaheman, Goldschmidt, Callaham, Delahunt, Nicolaou, Champion, Loiseau, Kutz, and Brunton}]{pysindy}
Kaptanoglu, A.~A.; de~Silva, B.; Fasel, U.; Kaheman, K.; Goldschmidt, A.; Callaham, J.; Delahunt, C.~B.; Nicolaou, Z.; Champion, K.~P.; Loiseau, J.; Kutz, J.~N.; and Brunton, S.~L. 2022.
\newblock PySINDy: {A} comprehensive Python package for robust sparse system identification.
\newblock \emph{J. Open Source Softw.}, 7(69): 3994.

\bibitem[{Kelley and Peterson(2001)}]{difference}
Kelley, W.~G.; and Peterson, A.~C. 2001.
\newblock \emph{Difference equations: an introduction with applications}.
\newblock Academic press.

\bibitem[{Kowshik et~al.(2021)Kowshik, Nagaraj, Jain, and Netrapalli}]{nlds_opt}
Kowshik, S.; Nagaraj, D.; Jain, P.; and Netrapalli, P. 2021.
\newblock Near-optimal offline and streaming algorithms for learning non-linear dynamical systems.
\newblock \emph{Advances in Neural Information Processing Systems}, 34: 8518--8531.

\bibitem[{La~Cava et~al.(2018)La~Cava, Singh, Taggart, Suri, and Moore}]{sr4}
La~Cava, W.; Singh, T.~R.; Taggart, J.; Suri, S.; and Moore, J.~H. 2018.
\newblock Learning concise representations for regression by evolving networks of trees.
\newblock In \emph{International Conference on Learning Representations}.

\bibitem[{Landajuela et~al.(2022)Landajuela, Lee, Yang, Glatt, Santiago, Aravena, Mundhenk, Mulcahy, and Petersen}]{sr2}
Landajuela, M.; Lee, C.~S.; Yang, J.; Glatt, R.; Santiago, C.~P.; Aravena, I.; Mundhenk, T.; Mulcahy, G.; and Petersen, B.~K. 2022.
\newblock A Unified Framework for Deep Symbolic Regression.
\newblock In \emph{Advances in Neural Information Processing Systems}, volume~35, 33985--33998.

\bibitem[{Li et~al.(2009)Li, McCann, Pollard, and Faloutsos}]{dynammo}
Li, L.; McCann, J.; Pollard, N.~S.; and Faloutsos, C. 2009.
\newblock DynaMMo: mining and summarization of coevolving sequences with missing values.
\newblock In \emph{Proceedings of the 15th {ACM} {SIGKDD} International Conference on Knowledge Discovery and Data Mining}, 507--516. {ACM}.

\bibitem[{Li, Prakash, and Faloutsos(2010)}]{plif}
Li, L.; Prakash, B.~A.; and Faloutsos, C. 2010.
\newblock Parsimonious Linear Fingerprinting for Time Series.
\newblock \emph{Proc. {VLDB} Endow.}, 3(1): 385--396.

\bibitem[{Liu and Hauskrecht(2015)}]{rlds}
Liu, Z.; and Hauskrecht, M. 2015.
\newblock A Regularized Linear Dynamical System Framework for Multivariate Time Series Analysis.
\newblock In \emph{Proceedings of the Twenty-Ninth {AAAI} Conference on Artificial Intelligence}, 1798--1804. {AAAI} Press.

\bibitem[{Mangan et~al.(2017)Mangan, Kutz, Brunton, and Proctor}]{aic2}
Mangan, N.~M.; Kutz, J.~N.; Brunton, S.~L.; and Proctor, J.~L. 2017.
\newblock Model selection for dynamical systems via sparse regression and information criteria.
\newblock \emph{Proceedings of the Royal Society A: Mathematical, Physical and Engineering Sciences}, 473(2204): 20170009.

\bibitem[{Matsubara and Sakurai(2019)}]{orbitmap}
Matsubara, Y.; and Sakurai, Y. 2019.
\newblock Dynamic Modeling and Forecasting of Time-evolving Data Streams.
\newblock In \emph{Proceedings of the 25th {ACM} {SIGKDD} International Conference on Knowledge Discovery {\&} Data Mining}, 458--468. {ACM}.

\bibitem[{Pavlovic, Rehg, and MacCormick(2000)}]{slds}
Pavlovic, V.; Rehg, J.~M.; and MacCormick, J. 2000.
\newblock Learning Switching Linear Models of Human Motion.
\newblock In \emph{Advances in Neural Information Processing Systems}, volume~13, 981--987. {MIT} Press.

\bibitem[{Pearl(1982)}]{sum_product}
Pearl, J. 1982.
\newblock Reverend bayes on inference engines: a distributed hierarchical approach.
\newblock In \emph{Proceedings of the Second AAAI Conference on Artificial Intelligence}, 133–136.

\bibitem[{Sattar and Oymak(2022)}]{nlds-opt2}
Sattar, Y.; and Oymak, S. 2022.
\newblock Non-asymptotic and Accurate Learning of Nonlinear Dynamical Systems.
\newblock \emph{Journal of Machine Learning Research}, 23(140): 1--49.

\bibitem[{Shojaee et~al.(2023)Shojaee, Meidani, Barati~Farimani, and Reddy}]{sr1}
Shojaee, P.; Meidani, K.; Barati~Farimani, A.; and Reddy, C. 2023.
\newblock Transformer-based Planning for Symbolic Regression.
\newblock In \emph{Advances in Neural Information Processing Systems}, volume~36, 45907--45919.

\bibitem[{Shor(1968)}]{grad}
Shor, N.~Z. 1968.
\newblock The rate of convergence of the generalized gradient descent method.
\newblock \emph{Cybernetics}, 4: 79--80.

\bibitem[{Sprott(2010)}]{halvorsen}
Sprott, J.~C. 2010.
\newblock \emph{Elegant chaos: algebraically simple chaotic flows}.
\newblock World Scientific.

\bibitem[{Udrescu and Tegmark(2020)}]{sr5}
Udrescu, S.-M.; and Tegmark, M. 2020.
\newblock AI Feynman: A physics-inspired method for symbolic regression.
\newblock \emph{Science Advances}, 6(16): eaay2631.

\bibitem[{Zarchan(2005)}]{ekf}
Zarchan, P. 2005.
\newblock \emph{Progress in astronautics and aeronautics: fundamentals of Kalman filtering: a practical approach}, volume 208.
\newblock Aiaa.

\bibitem[{Zou and Hastie(2005)}]{regularization}
Zou, H.; and Hastie, T. 2005.
\newblock Regularization and variable selection via the elastic net.
\newblock \emph{Journal of the Royal Statistical Society Series B: Statistical Methodology}, 67(2): 301--320.

\end{thebibliography}
\clearpage
\appendix
\section*{Appendix}
\section{Symbols and Definitions}
\begin{table}[H]
\centering
\caption{Symbols and definitions.}
\label{table: define}
\scalebox{0.80}{
\begin{tabular}{l|l}
\toprule
Symbol & Definition \\
\midrule
${\datalen}$
    & Number of time points\\
${\datadim}$
    & Number of dimensions \\
${\sdim}$
    & Number of state dimensions \\
${\nldim}$
    & Order of \method\\
${\nlsdim}$
    & Number of non-linear state dimensions\\ 
${\data}$
    & $d$-dimentional time series, i.e., $\data = \{\datavec(1), . . . , \datavec(\datalen)\}$ \\
${\datavec(t)}$
    & $d$-dimensional vector at time point $t$, i.e., $\datavec(t) = \{\datapoint(t)\}^{\datadim}_{i=1}$ \\
\midrule
${\statevec(t)}$
    & Latent states at time point $t$, i.e., $\statevec(t) = \{\statepoint(t)\}^{\sdim}_{i=1}$ \\
${\nlstatevec}$
    & Non-Linearities, i.e., $\nonlinearities$\\
${\estvec(t)}$
    & Estimated values at time point $t$, i.e., $\estvec(t) = \{\estpoint(t)\}^{\datadim}_{i=1}$ \\
\midrule
${\lMat}$
    & Linear transition matrix, ${\sdim}\times{\sdim}$\\
${\nlMat}$
    & Non-linear tansition matrix, ${\sdim}\times{\nlsdim}$\\
${\sOf}$
    & $\sdim$-dimensional offset at state space\\
${\oMat}$
    & Observation matrix, ${\datadim}\times{\sdim}$\\
${\oOf}$
    & $\datadim$-dimensional observation offset\\
\bottomrule
\end{tabular}
}
\normalsize
\end{table}
Table \ref{table: define} lists the symbols and definitions used in this paper.
\section{Minimum Description Length (MDL) Cost}
\label{apseq: mdl}
Our method defines the minimum description length (MDL) cost as follows.

\begin{align}
    \underbrace{<\data;\setparams>}_{\text{Total MDL cost}} = \underbrace{<\data|\setparams>}_{\text{Data encoding cost}} + \underbrace{<\setparams>}_{\text{Model description cost}}.
\end{align}

Data encoding cost, $<\data|\cdot>$, measures how well a given model compresses original data, for which
the Huffman coding scheme \cite{mdl} assigns a number of bits to each element in $\data$. Letting
$\hat{\datapoint} \in \hat{\data}$ be reconstructed data, the data encoding cost is represented by the negative log-likelihood under a Gaussian distribution with mean $\mu$ and variance $\sigma^2$ over errors, as follows.

\begin{align}
    <\data|\setparams>=\sum_{\datapoint \in \data} -\log_2 P_{\mu,\sigma}(\datapoint-\hat{\datapoint}).
\end{align}

Model description cost, $<\setparams>$, measures the complexity of a given model, which is defined as follows.
\begin{align}
    \nonumber
    <\setparams> = <\lMat> + <\nlMat> + <\sOf> + <\oMat> +<\oOf>.
\end{align}
Each cost measures the complexity based on the following equations.
\begin{align}
    \nonumber
    <\lMat> &= |\lMat|\cdot(\log(\sdim)+\log(\sdim)+c_F), \\
    \nonumber
    <\nlMat> &= |\nlMat|\cdot(\log(\sdim)+\log(\nlsdim)+c_F),\\
    \nonumber
    <\sOf> &= |\sOf|\cdot(\log(\sdim)+c_F), \\
    \nonumber
    <\oMat> &= \sdim\datadim\cdot(\log(\sdim)+\log(\datadim)+c_F),\\
    \nonumber
    <\oOf> &= \datadim\cdot(\log(\datadim)+c_F),
\end{align}

where $|\cdot|$ shows the number of non-zero elements in a given vector/matrix, and $c_F$ is the float point cost\footnote{We used $c_F=32$ bits.}. 
\section{Proof of Proposition 1}
\begin{proof}
\label{apseq: proof}
Given that the gradient of $g(\sparseMat)$ is Eq. \eqref{eq: grad} and since the Frobenius norm is submultiplicative, the following inequality holds:
\begin{align}
    \nonumber
    &\|\nabla g(X)-\nabla g(Y )\|_F, \\
    \nonumber
    =&\|\sCov^{-1}\sum^{\datalen-1}_{t=1} (X-Y)\mathbb{E}[\statevec_\phi(t)\statevec_\phi^T(t)]\|_F + \lambda_2||X-Y||_F,
    \\
    \nonumber
    \leq& (\|\sCov^{-1}\|_F\cdot\|\sum^{\datalen-1}_{t=1}\mathbb{E}[\statevec_\phi(t)\statevec_\phi^T(t)]\|_F + \lambda_2)\cdot\|X-Y\|_F.
\end{align}
In Eq. \eqref{eq: s_obs}, $f(\sparseMat) = g(\sparseMat) + \lambda_1\|\sparseMat - \eye\|_1$, and $g(\sparseMat)$ has a Lipschitz continuous gradient with the constant:
\begin{align}
    \nonumber \|\sCov^{-1}\|_F\cdot\|\sum^{\datalen-1}_{t=1}\mathbb{E}[\statevec_\phi(t)\statevec_\phi^T(t)]\|_F + \lambda_2.
\end{align} 
Under this condition and according to \cite{grad,grad2}, the following holds:
\begin{align}
    f(\sparseMat(i))-f(\sparseMat^{opt}) \leq \frac{\|\sparseMat(0)-\sparseMat^{opt}\|_F^2} {2\alpha i},
\end{align}
where $\sparseMat(0)$ represents the initial matrices and $\sparseMat^{opt}$ represents the optimal matrices, with $i$ indicating the number of iterations.
\end{proof}
\section{Fitting Non-Sparse Parameter Set}
\label{apseq: fit_nonsparse}
To estimate the parameter $\oMat$, taking the derivatives of \eqref{eq:all_ob} with respect at the components of $\oMat^{new}$ and setting them to zero yields the following results:
\begin{align}
    \label{eq: ns_st}
    \oMat^{new} &= (\sum^\datalen_{t=1} \datavec(t)\mathbb{E}[\statevec^T(t)])(\sum^\datalen_{t=1} \mathbb{E}[\statevec(t)\statevec^T(t)])^{-1}.
\end{align}
We also update the state/observation offsets $\sOf/\oOf$:
\begin{align}
    \sOf^{new} &= \frac{1}{\datalen-1}\sum^{\datalen-1}_{t=1} \{\mathbb{E}[{\statevec}(t+1)]- \sparseMat^{new}\mathbb{E}[{\statevec}_\phi(t)]\},\\
    \label{eq: ns_ed}
    \oOf^{new} &= \frac{1}{\datalen}\sum^\datalen_{t=1} \{\datavec(t)-\oMat^{new}\mathbb{E}[{\statevec}(t)]\}.
\end{align}
Finally, we estimate $\sCov, \oCov$.
\begin{align}
    \nonumber
    \sCov^{new} &= \frac{1}{\datalen-1}\sum^{\datalen-1}_{t=1}\mathbb{E}[(\statevec(t+1)-\jMat_{\mathbb{E}[\statevec(t)]}(\statevec(t)-\mathbb{E}[\statevec(t)])\\
    \nonumber
    &-\sparseMat^{new}\mathbb{E}[\statevec_\phi(t)]-\sOf^{new})(\statevec(t+1)-\jMat_{\mathbb{E}[\statevec(t)]}(\statevec(t)-\mathbb{E}[\statevec(t)]) \\
    &-\sparseMat^{new}\mathbb{E}[\statevec_\phi(t)]-\sOf^{new})^T], \\
    \nonumber
    \oCov^{new} &= \frac{1}{\datalen}\sum^\datalen_{t=1}\mathbb{E}[(\datavec(t)-\oMat^{new}\statevec(t)-\oOf^{new})\\
    &(\datavec(t)-\oMat^{new}\statevec(t)-\oOf^{new})^T].
\end{align}

\begin{figure*}[ht]
    \centering
    \includegraphics[width=1.0\linewidth]{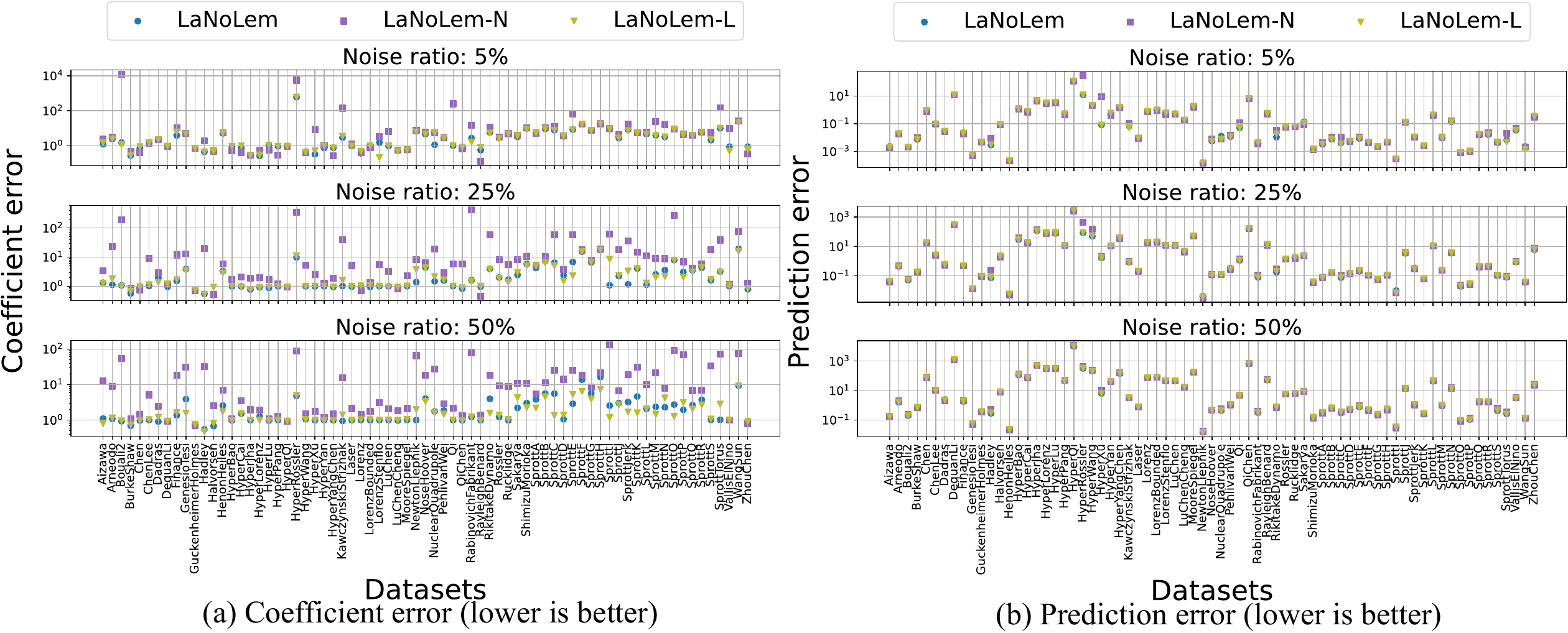}
    \vspace{-1.5em}
    \caption{Ablation study: The regularizer in \method improves (a) the estimation and (b) prediction accuracy of the system.}
    \label{apfig: exp_syn}
    \vspace{-2.0em}
\end{figure*}

\section{Experimental Setting}
\label{apseq: setting}
\mypara{Computing infrastructure}
The configuration includes 2 * Xeon Gold 6444Y 3.6GHz CPU, 12 * 64GB DDR4 RAM (768GB), 
which is sufficient for all the baselines.

\mypara{Reproduction details for \method}
Our \method optimization problem has a regularization parameter $\lambda_1$ and $\lambda_2$, which determines the sparsity level in the parameter. 
In addition to this, the order $\nldim$ and the number of state dimensions $\sdim$ must be given. 
For each experiment, 
the parameters were selected as follows.
We searched for the best results in $\lambda_1, \lambda_2 \in \{0.0, 1.0, 10, 50, 100, 500\}$ and $\nldim \in \{2,3,4\}$ to minimize MDL. We also set $\sdim = \datadim$ to make the experimental conditions the same as those in other methods used in experiments for the synthetic dataset.

\mypara{Baselines}
For the \sindy framework, as no official code is available, we used a modified version of a third-party implementation found at https://github.com/dynamicslab/pysindy \cite{pysindy}, and we used sequentially thresholded least squares (STLSQ), step-wise sparse regression (SSR) and mixed-integer optimized sparse regression (MIOSR) to optimize the parameters \cite{sindy,ssr,mio-sindy}. For MIOSR, STLSQ, and SSR, we tuned over the regularization strength $\alpha \in \{0, 10^{-5}, 10^{-3}, 0.01, 0.05, 0.2\}$. In \cite{mio-sindy}, hyperparameters are selected 
when minimizing the Akaike information criterion (AIC) metric \cite{aic,aic2}. We also selected all hyperparameters accordingly.

\mypara{Synthetics details}
We used the benchmark dataset introduced by \cite{dysts}. The length of the data used to estimate the system is 500 steps, with a time step of 0.01 between each sample. Additional data of length 100, which was not used for training, was generated to evaluate the accuracy of the prediction. Initial conditions were generated according to \cite{dysts}.
\section{Ablation Study}
\label{apseq: ablation}
In this section, ablation experiments are performed to evaluate the performance of the proposed regularizer. 
For a detailed evaluation, we use \methodl and \methodn, which are limited versions of the proposed method. \methodl uses only the $l_1$ norm as the regularizer in Eq. \eqref{eq:all_ob}, while LaNoLem-N does not use the regularizer $r(\cdot)$ in Eq. \eqref{eq:all_ob}.
Fig. \ref{apfig: exp_syn} shows a comparative analysis of \method, \methodl, and \methodn, evaluated under the same experimental conditions as in Section \ref{sec: 04_exp}. Table \ref{ap_table: win_count} summarizes the number of wins associated with each method. The results show that the inclusion of regularization improves both estimation and prediction accuracy. Specifically, at high noise ratios, the $l_1$ norm is shown to improve the accuracy of dynamics estimation, while the Frobenius norm improves prediction accuracy. \begin{table}[t]
\vspace{-1.0em}
\hspace{-1.0em}
\caption{$1^{\text{st}}$ Count in ablation study (higher is better).}
\label{ap_table: win_count}
\vspace{0.5em}
\centering
\scalebox{0.9}
{
\begin{tabular}{lcccccc}
\toprule
\large \textbf{Measures} & \multicolumn{3}{c}{\large \textbf{Coefficient error}} & \multicolumn{3}{c}{\textbf{Prediction error}}\\
\cmidrule(lr){1-1} \cmidrule(lr){2-4} \cmidrule(lr){5-7}
\large \textbf{Noise ratio} & 5\% & 25\% & 50\%  & 5\% & 25\%& 50\%\\
\midrule
\large \textbf{\methodn} & 16 & 4 & 4 & \textbf{43} & 21  & 16\\
\large \textbf{\methodl} & 27 & 33 & \textbf{40} & 17 & 21  & 21\\
\large \textbf{\method} & \textbf{28} & \textbf{34} & 27 & 11 & \textbf{29} & \textbf{34}\\
\bottomrule
\end{tabular}
}
\vspace{-1.0em}
\end{table}

\section{Full Experimental Results}
\label{apseq: full result}
Table \ref{ap_table: fullresult-1}-\ref{ap_table: fullresult-2} shows the mean and standard deviation of the Coefficient error and Prediction error for \method and its baselines for dysts dataset, with the best performance shown in \textbf{bold}.
As discussed in Section \ref{sec: 04_exp}, our method is competitive in terms of coefficient errors and consistently achieves the best prediction error. Our method achieved low coefficient error because it can filter out background noise when estimating systems using the temporal dependency of the data, while other methods can't.
\begin{table*}
\centering
\caption{Noise ratio : 5\%}
\label{ap_table: fullresult-1}
\scalebox{0.7}
{
\begin{tabular}{lllllllll}
\toprule
 & \multicolumn{2}{c}{Aizawa} & \multicolumn{2}{c}{Arneodo} & \multicolumn{2}{c}{Bouali2} & \multicolumn{2}{c}{BurkeShaw} \\
 & Coefficient error & Prediction error & Coefficient error & Prediction error & Coefficient error & Prediction error & Coefficient error & Prediction error \\
\midrule
LaNoLem & $\mathbf{1.216}\pm 0.568$ & $0.002\pm 0.001$ & $2.416\pm 1.37$ & $0.019\pm 0.001$ & $1.529\pm 0.13$ & $0.002\pm 0.0$ & $0.272\pm 0.009$ & $0.007\pm 0.001$ \\
LaNoLem-L & $1.485\pm 0.003$ & $0.002\pm 0.0$ & $2.419\pm 1.323$ & $0.019\pm 0.002$ & $\mathbf{1.144}\pm 0.154$ & $\mathbf{0.002}\pm 0.0$ & $0.272\pm 0.006$ & $\mathbf{0.007}\pm 0.001$ \\
LaNoLem-N & $2.472\pm 0.822$ & $\mathbf{0.002}\pm 0.0$ & $3.093\pm 0.94$ & $\mathbf{0.018}\pm 0.002$ & $12008.814\pm 26408.657$ & $0.002\pm 0.0$ & $0.48\pm 0.133$ & $0.01\pm 0.001$ \\
MIOSR & $1.513\pm 0.01$ & $0.003\pm 0.0$ & $\mathbf{2.013}\pm 0.037$ & $0.031\pm 0.002$ & $66.789\pm 66.822$ & $0.004\pm 0.0$ & $\mathbf{0.229}\pm 0.003$ & $0.012\pm 0.001$ \\
SSR & $1.643\pm 0.015$ & $0.003\pm 0.0$ & $2.781\pm 0.34$ & $0.03\pm 0.002$ & $13.544\pm 7.776$ & $0.004\pm 0.0$ & $0.244\pm 0.011$ & $0.012\pm 0.001$ \\
STLSQ & $1.636\pm 0.005$ & $0.003\pm 0.0$ & $2.722\pm 0.294$ & $0.03\pm 0.002$ & $42.937\pm 62.935$ & $0.004\pm 0.0$ & $0.234\pm 0.011$ & $0.012\pm 0.001$ \\
\bottomrule
\end{tabular}
}
\end{table*}
\begin{table*}
\centering
\caption{Noise ratio : 5\%}
\scalebox{0.7}
{
\begin{tabular}{lllllllll}
\toprule
 & \multicolumn{2}{c}{NuclearQuadrupole} & \multicolumn{2}{c}{PehlivanWei} & \multicolumn{2}{c}{Qi} & \multicolumn{2}{c}{QiChen} \\
 & Coefficient error & Prediction error & Coefficient error & Prediction error & Coefficient error & Prediction error & Coefficient error & Prediction error \\
\midrule
LaNoLem & $\mathbf{1.136}\pm 0.034$ & $0.013\pm 0.002$ & $2.752\pm 0.017$ & $0.012\pm 0.001$ & $1.003\pm 0.001$ & $0.051\pm 0.005$ & $0.817\pm 0.122$ & $6.639\pm 0.491$ \\
LaNoLem-L & $5.357\pm 0.386$ & $0.008\pm 0.004$ & $2.721\pm 0.018$ & $\mathbf{0.011}\pm 0.001$ & $1.004\pm 0.0$ & $\mathbf{0.05}\pm 0.002$ & $0.817\pm 0.126$ & $6.658\pm 0.501$ \\
LaNoLem-N & $5.52\pm 0.248$ & $0.008\pm 0.003$ & $2.912\pm 0.112$ & $0.015\pm 0.001$ & $249.217\pm 532.112$ & $0.117\pm 0.145$ & $\mathbf{0.683}\pm 0.077$ & $\mathbf{6.446}\pm 0.459$ \\
MIOSR & $3.987\pm 0.32$ & $0.01\pm 0.002$ & $\mathbf{2.594}\pm 0.017$ & $0.022\pm 0.001$ & $\mathbf{1.002}\pm 0.003$ & $0.084\pm 0.006$ & $0.913\pm 0.101$ & $12.163\pm 0.715$ \\
SSR & $5.439\pm 0.172$ & $0.009\pm 0.001$ & $2.672\pm 0.015$ & $0.019\pm 0.001$ & $2.41\pm 0.71$ & $0.084\pm 0.005$ & $1.451\pm 0.147$ & $12.138\pm 0.708$ \\
STLSQ & $5.134\pm 0.233$ & $\mathbf{0.007}\pm 0.001$ & $2.702\pm 0.032$ & $0.019\pm 0.001$ & $2.344\pm 0.737$ & $0.084\pm 0.005$ & $1.364\pm 0.133$ & $12.157\pm 0.692$ \\
\bottomrule
\end{tabular}
}
\end{table*}

\begin{table*}
\centering
\caption{Noise ratio : 25\%}
\label{}
\scalebox{0.7}
{
\begin{tabular}{lllllllll}
\toprule
 & \multicolumn{2}{c}{Aizawa} & \multicolumn{2}{c}{Arneodo} & \multicolumn{2}{c}{Bouali2} & \multicolumn{2}{c}{BurkeShaw} \\
 & Coefficient error & Prediction error & Coefficient error & Prediction error & Coefficient error & Prediction error & Coefficient error & Prediction error \\
\midrule
LaNoLem & $1.316\pm 0.04$ & $0.04\pm 0.006$ & $\mathbf{1.117}\pm 0.079$ & $0.429\pm 0.033$ & $1.067\pm 0.153$ & $\mathbf{0.051}\pm 0.01$ & $0.562\pm 0.038$ & $0.177\pm 0.028$ \\
LaNoLem-L & $\mathbf{1.263}\pm 0.035$ & $0.041\pm 0.006$ & $1.804\pm 0.729$ & $\mathbf{0.409}\pm 0.033$ & $\mathbf{0.983}\pm 0.086$ & $0.051\pm 0.009$ & $0.662\pm 0.008$ & $0.186\pm 0.027$ \\
LaNoLem-N & $3.329\pm 1.024$ & $\mathbf{0.039}\pm 0.006$ & $23.174\pm 20.294$ & $0.498\pm 0.145$ & $192.743\pm 199.623$ & $0.061\pm 0.012$ & $0.866\pm 0.607$ & $\mathbf{0.176}\pm 0.021$ \\
MIOSR & $1.417\pm 0.122$ & $0.07\pm 0.008$ & $1.285\pm 0.111$ & $0.742\pm 0.036$ & $33.514\pm 39.735$ & $0.091\pm 0.011$ & $\mathbf{0.397}\pm 0.124$ & $0.299\pm 0.027$ \\
SSR & $1.425\pm 0.239$ & $0.07\pm 0.008$ & $1.843\pm 0.549$ & $0.738\pm 0.038$ & $8.281\pm 4.342$ & $0.092\pm 0.011$ & $0.555\pm 0.046$ & $0.303\pm 0.029$ \\
STLSQ & $1.425\pm 0.239$ & $0.07\pm 0.008$ & $1.803\pm 0.591$ & $0.738\pm 0.037$ & $8.274\pm 4.348$ & $0.092\pm 0.011$ & $0.554\pm 0.047$ & $0.303\pm 0.029$ \\
\bottomrule
\end{tabular}
}
\end{table*}
\begin{table*}
\centering
\caption{Noise ratio : 25\%}
\scalebox{0.7}
{
\begin{tabular}{lllllllll}
\toprule
 & \multicolumn{2}{c}{NuclearQuadrupole} & \multicolumn{2}{c}{PehlivanWei} & \multicolumn{2}{c}{Qi} & \multicolumn{2}{c}{QiChen} \\
 & Coefficient error & Prediction error & Coefficient error & Prediction error & Coefficient error & Prediction error & Coefficient error & Prediction error \\
\midrule
LaNoLem & $\mathbf{1.478}\pm 0.787$ & $0.126\pm 0.01$ & $\mathbf{1.575}\pm 0.937$ & $0.344\pm 0.075$ & $1.003\pm 0.001$ & $\mathbf{1.15}\pm 0.058$ & $\mathbf{0.831}\pm 0.098$ & $\mathbf{161.471}\pm 12.079$ \\
LaNoLem-L & $2.132\pm 0.058$ & $\mathbf{0.12}\pm 0.01$ & $1.661\pm 0.259$ & $0.293\pm 0.033$ & $1.002\pm 0.0$ & $1.176\pm 0.059$ & $0.875\pm 0.119$ & $163.411\pm 12.288$ \\
LaNoLem-N & $18.783\pm 18.923$ & $0.127\pm 0.02$ & $2.863\pm 0.084$ & $\mathbf{0.275}\pm 0.022$ & $5.748\pm 1.49$ & $1.359\pm 0.066$ & $5.776\pm 10.267$ & $172.999\pm 19.589$ \\
MIOSR & $3.341\pm 0.124$ & $0.168\pm 0.012$ & $2.233\pm 0.108$ & $0.477\pm 0.012$ & $\mathbf{0.993}\pm 0.019$ & $2.105\pm 0.137$ & $1.032\pm 0.057$ & $306.702\pm 15.217$ \\
SSR & $4.037\pm 0.223$ & $0.169\pm 0.011$ & $2.279\pm 0.116$ & $0.478\pm 0.012$ & $1.551\pm 0.399$ & $2.128\pm 0.143$ & $1.093\pm 0.052$ & $304.602\pm 20.215$ \\
STLSQ & $4.036\pm 0.224$ & $0.169\pm 0.011$ & $2.278\pm 0.12$ & $0.478\pm 0.012$ & $1.551\pm 0.399$ & $2.128\pm 0.143$ & $1.094\pm 0.064$ & $304.446\pm 20.259$ \\
\bottomrule
\end{tabular}
}
\end{table*}

\begin{table*}
\centering
\caption{Noise ratio : 50\%}
\scalebox{0.7}
{
\begin{tabular}{lllllllll}
\toprule
 & \multicolumn{2}{c}{Aizawa} & \multicolumn{2}{c}{Arneodo} & \multicolumn{2}{c}{Bouali2} & \multicolumn{2}{c}{BurkeShaw} \\
 & Coefficient error & Prediction error & Coefficient error & Prediction error & Coefficient error & Prediction error & Coefficient error & Prediction error \\
\midrule
LaNoLem & $1.102\pm 0.273$ & $\mathbf{0.155}\pm 0.024$ & $\mathbf{1.089}\pm 0.094$ & $\mathbf{1.6}\pm 0.089$ & $\mathbf{0.931}\pm 0.127$ & $\mathbf{0.197}\pm 0.026$ & $\mathbf{0.691}\pm 0.015$ & $0.699\pm 0.077$ \\
LaNoLem-L & $\mathbf{0.788}\pm 0.222$ & $0.166\pm 0.029$ & $1.142\pm 0.098$ & $1.668\pm 0.046$ & $0.947\pm 0.091$ & $0.198\pm 0.03$ & $0.728\pm 0.032$ & $0.733\pm 0.075$ \\
LaNoLem-N & $12.653\pm 11.711$ & $0.171\pm 0.031$ & $8.91\pm 1.237$ & $1.986\pm 0.124$ & $54.548\pm 12.426$ & $0.235\pm 0.027$ & $1.089\pm 0.769$ & $\mathbf{0.687}\pm 0.06$ \\
MIOSR & $2.314\pm 0.864$ & $0.279\pm 0.031$ & $1.266\pm 0.328$ & $3.05\pm 0.217$ & $16.407\pm 12.539$ & $0.366\pm 0.045$ & $0.82\pm 0.097$ & $1.217\pm 0.12$ \\
SSR & $1.486\pm 0.353$ & $0.281\pm 0.033$ & $1.913\pm 0.433$ & $2.957\pm 0.153$ & $7.203\pm 3.297$ & $0.366\pm 0.046$ & $0.758\pm 0.053$ & $1.217\pm 0.116$ \\
STLSQ & $1.486\pm 0.353$ & $0.281\pm 0.033$ & $1.916\pm 0.438$ & $2.957\pm 0.152$ & $7.202\pm 3.298$ & $0.366\pm 0.046$ & $0.759\pm 0.053$ & $1.217\pm 0.116$ \\
\bottomrule
\end{tabular}
}
\end{table*}
\begin{table*}
\centering
\caption{Noise ratio : 50\%}
\label{ap_table: fullresult-2}
\scalebox{0.7}
{
\begin{tabular}{lllllllll}
\toprule
 & \multicolumn{2}{c}{NuclearQuadrupole} & \multicolumn{2}{c}{PehlivanWei} & \multicolumn{2}{c}{Qi} & \multicolumn{2}{c}{QiChen} \\
 & Coefficient error & Prediction error & Coefficient error & Prediction error & Coefficient error & Prediction error & Coefficient error & Prediction error \\
\midrule
LaNoLem & $1.75\pm 0.643$ & $\mathbf{0.438}\pm 0.033$ & $1.851\pm 0.102$ & $1.078\pm 0.111$ & $1.002\pm 0.001$ & $\mathbf{4.582}\pm 0.238$ & $0.98\pm 0.029$ & $\mathbf{676.984}\pm 52.033$ \\
LaNoLem-L & $\mathbf{1.596}\pm 0.043$ & $0.442\pm 0.032$ & $\mathbf{1.47}\pm 0.07$ & $1.14\pm 0.104$ & $1.002\pm 0.001$ & $4.612\pm 0.286$ & $\mathbf{0.953}\pm 0.037$ & $677.656\pm 48.206$ \\
LaNoLem-N & $27.046\pm 13.17$ & $0.548\pm 0.112$ & $2.9\pm 0.109$ & $\mathbf{1.057}\pm 0.094$ & $2.192\pm 0.58$ & $5.087\pm 0.271$ & $1.323\pm 0.404$ & $699.366\pm 57.783$ \\
MIOSR & $2.28\pm 0.281$ & $0.665\pm 0.046$ & $1.555\pm 0.353$ & $1.905\pm 0.051$ & $\mathbf{1.002}\pm 0.002$ & $8.382\pm 0.532$ & $1.0\pm 0.0$ & $1240.499\pm 77.713$ \\
SSR & $3.215\pm 0.201$ & $0.669\pm 0.043$ & $2.201\pm 0.387$ & $1.909\pm 0.053$ & $1.243\pm 0.272$ & $8.492\pm 0.545$ & $1.027\pm 0.063$ & $1216.263\pm 82.011$ \\
STLSQ & $3.213\pm 0.201$ & $0.669\pm 0.043$ & $2.203\pm 0.389$ & $1.91\pm 0.053$ & $1.242\pm 0.272$ & $8.494\pm 0.546$ & $1.028\pm 0.035$ & $1225.728\pm 77.02$ \\
\bottomrule
\end{tabular}
}
\end{table*}

\end{document}